\documentclass{article}

\RequirePackage{expl3}
\ExplSyntaxOn
\ExplSyntaxOff

\usepackage[preprint]{neurips_2025} 

\usepackage[utf8]{inputenc} 
\usepackage[T1]{fontenc}    
\usepackage{hyperref}       
\usepackage{url}            
\usepackage{booktabs}       
\usepackage{amsfonts}       
\usepackage{nicefrac}       
\usepackage{microtype}      
\usepackage{xcolor}         

\usepackage{graphicx}
\usepackage{float}

\usepackage{tabularray}
\UseTblrLibrary{booktabs}

\IfFileExists{preamble.tex}{\usepackage{bm}
\usepackage{amsmath}
\usepackage{amssymb}
\usepackage{amsthm}
\usepackage[scr=zapfc]{mathalpha}
\usepackage{cleveref}
\usepackage{tabularray}
\usepackage{graphicx}
\usepackage{subcaption}
\usepackage{float}
\usepackage{wrapfig}
\usepackage{pifont}
\usepackage{lipsum}
\usepackage{marvosym}
\usepackage{enumitem}

\renewcommand{\d}{\mathrm{d}}

\theoremstyle{plain}

\newtheorem{proposition}{Proposition}
}{}

\title{Riemannian Motion Generation \\ {\large A Unified Framework for Human Motion Representation and Generation via Riemannian Flow Matching}}

\author{%
  \textbf{Fangran Miao}$^{1}$
  ~~
  \textbf{Jian Huang}$^{1,2}$$\textsuperscript{\Letter}$
  ~~
  \textbf{Ting Li}$^{3}$$\textsuperscript{\Letter}$ \thanks{Email: \texttt{\href{mailto:fangran.miao@connect.polyu.hk}{fangran.miao@connect.polyu.hk}, \href{mailto:j.huang@polyu.edu.hk}{j.huang@polyu.edu.hk}, \href{mailto:liting@sustech.edu.cn}{liting@sustech.edu.cn}}}\\
  {\small $^{1}$Department of Data Science and Artificial Intelligence, The Hong Kong Polytechnic University}\\
  {\small $^{2}$Department of Applied Mathematics, The Hong Kong Polytechnic University}\\
  {\small $^{3}$Department of Data Science and Statistics, Southern University of Science and Technology}\\
  $\textsuperscript{\Letter}$Corresponding Author \\
  \raisebox{-0.25em}{\includegraphics[height=1.15em]{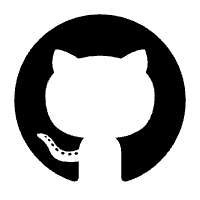}}\ \href{https://frank-miao.github.io/RMG-Project-Page}{\texttt{Project Page}}
}

\begin{document}

\maketitle

\begin{figure}[H]
  \vspace{-1.0\baselineskip}
  \centering
  \begin{minipage}{0.34\textwidth}
    \centering
    \includegraphics[width=\linewidth]{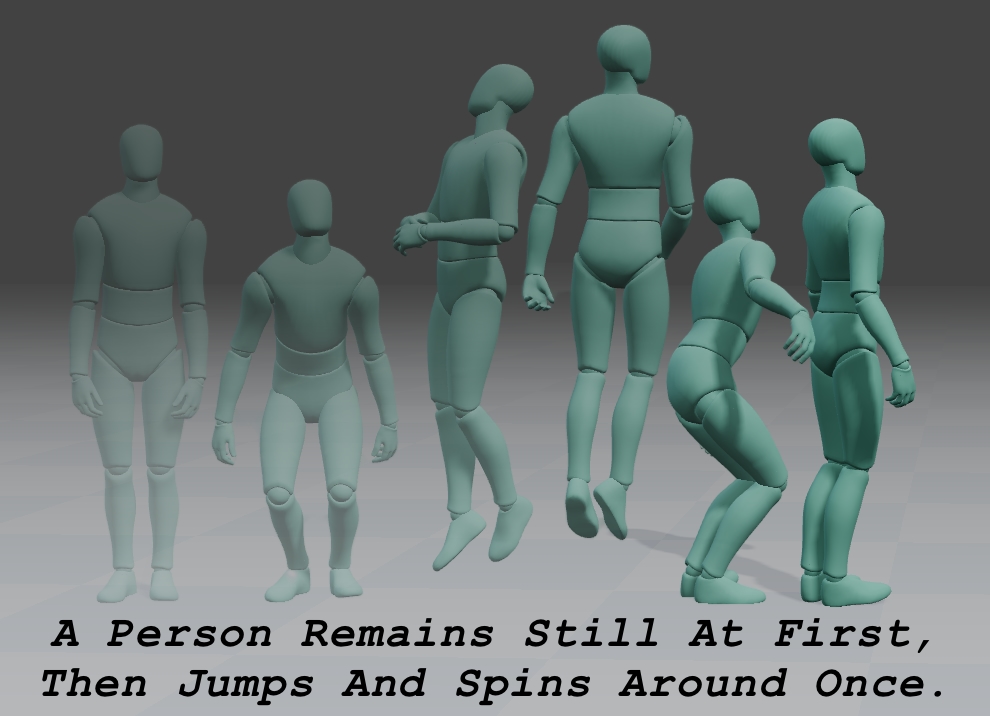}
  \end{minipage}
  \hfill
  \begin{minipage}{0.65\textwidth}
    \centering
    \includegraphics[width=\linewidth]{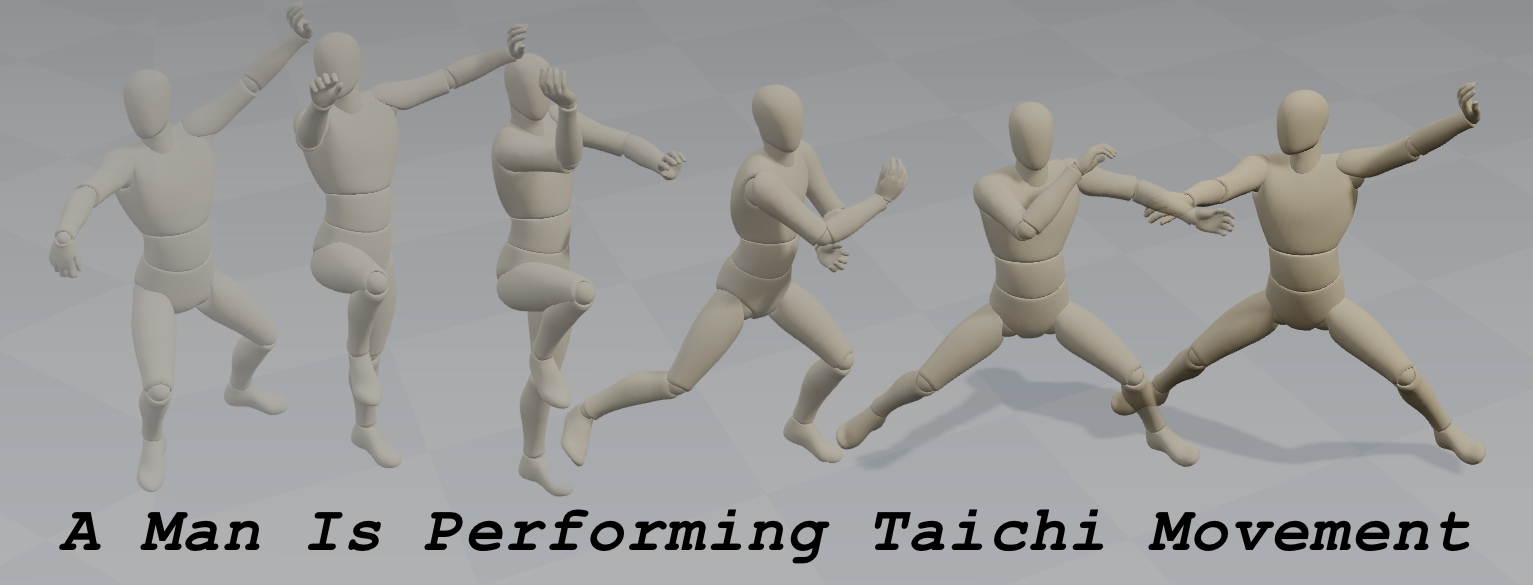}
  \end{minipage}

  \vspace{0.5em}

  \begin{minipage}{0.99\textwidth}
    \centering
    \includegraphics[width=\linewidth]{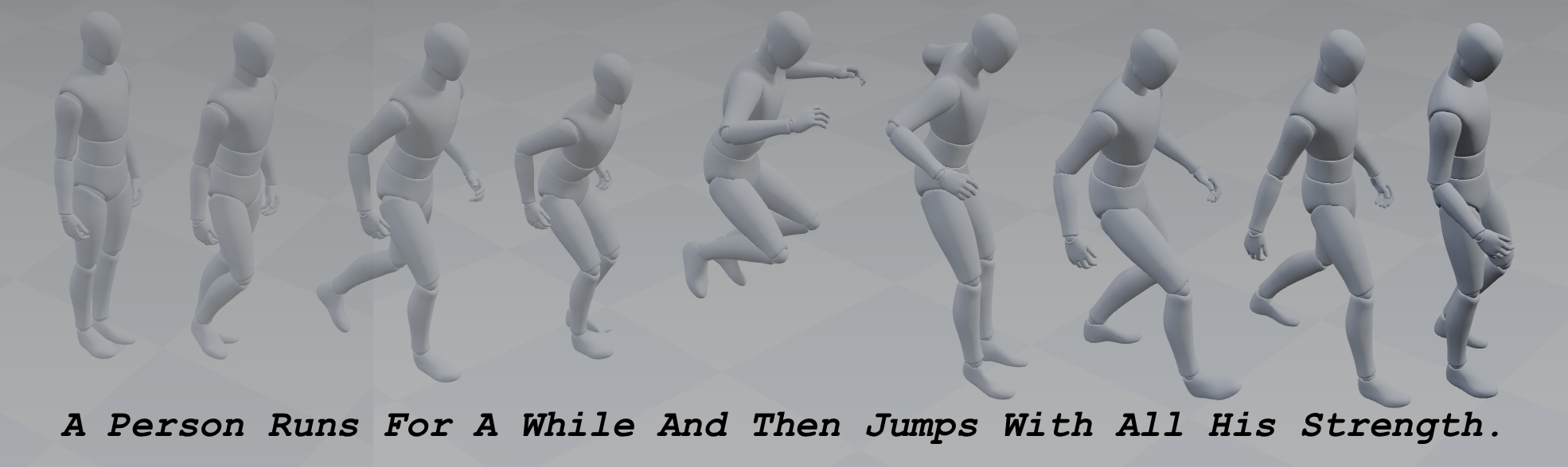}
  \end{minipage}
  \caption{Text-to-motion samples under our \textbf{Riemannian Motion Generation} framework.}
  \label{fig:teaser}
\end{figure}

\begin{abstract}
  Human motion generation is often learned in Euclidean spaces, although valid motions follow structured non-Euclidean geometry. We present \textbf{Riemannian Motion Generation (RMG)}, a unified framework that represents motion on a product manifold and learns dynamics via Riemannian flow matching. RMG factorizes motion into several manifold factors, yielding a scale-free representation with intrinsic normalization, and uses geodesic interpolation, tangent-space supervision, and manifold-preserving ODE integration for training and sampling. On HumanML3D, RMG achieves state-of-the-art FID in the HumanML3D format (0.043) and ranks first on all reported metrics under the MotionStreamer format. On MotionMillion, it also surpasses strong baselines (FID 5.6, R@1 0.86).
Ablations show that the compact $\mathscr{T}+\mathscr{R}$ (translation + rotations) representation is the most stable and effective, highlighting geometry-aware modeling as a practical and scalable route to high-fidelity motion generation.

\end{abstract}

\section{Introduction}

Conditional human motion generation has emerged as a key challenge in generative modeling, incorporating conditioning signals that range from text descriptions and action labels to audio, music, and scene context \citep{motion_survey}. Synthesizing high-fidelity motion sequences is essential for advancements in human-computer interaction, embodied AI, and augmented-reality content creation.

Recent progress in human motion generation has primarily focused on model architecture. While earlier methods relied on VAEs (e.g., TEMOS \citep{petrovich2022temos} and T2M \citep{humanml3d}), state-of-the-art systems increasingly adopt diffusion or autoregressive frameworks \citep{motion_diffuse,mdm,motionstreamer,motionlab,kim2023flame,momask,t2m_gpt,mgpt}. By contrast, the geometry of motion representation has received less systematic attention, despite its direct impact on optimization difficulty, sampling stability, and physical plausibility.

Most existing pipelines still encode motion using redundant Euclidean coordinates, enforcing validity only implicitly through constraints or post-processing. For a skeleton with
$J$ joints, an articulated pose has intrinsic degrees of freedom on the order of
$3J;$ however, common encodings concatenate multiple correlated views of the same state, occupying a much higher-dimensional ambient space. \Cref{tab:format_comparison} summarizes these encoding strategies across previous methods. Consequently, models are trained in
$\mathbb{R}^D,$ whereas physically valid motions reside on, or near, a lower-dimensional manifold with intrinsic dimension $d.$

This observation suggests a simple but important motivation: if data already resides on a lower-dimensional manifold and follows a strong geometric structure, the generative models should benefit from operating in that space and respect that structure rather than in an unconstrained ambient vector space.
Our key insight is that much of the apparent complexity of human motion does not come from arbitrary high-dimensional variation, but from composing several low-dimensional factors, each with its own natural geometry.
Once this factorization is made explicit, geometric consistency no longer needs to be imposed only after generation.
Instead, it can be built directly into both the representation and the generative dynamics.
This viewpoint naturally leads to a geometry-aware formulation in which representation and model design are developed together, rather than treated as separate concerns.

Guided by this perspective, we first provide a unified geometric view that decomposes existing representations into common factors and their natural manifolds. Building on this view, we introduce \textbf{Riemannian Motion Generation (RMG)}, a representation-and-generation framework that models motion on a product manifold and learns dynamics via Riemannian flow matching.
RMG yields a compact manifold-aware parameterization and a geometry-consistent training/inference pipeline.
On HumanML3D text-to-motion benchmarks, it matches or exceeds strong baselines across quality, alignment, and diversity metrics.
Moreover, RMG surpasses all baselines on the large-scale MotionMillion dataset, demonstrating superior scalability and generalization.

We summarize our core contributions as follows:

\begin{itemize}
  \item We propose \textbf{Riemannian Motion Generation (RMG)}, a geometric paradigm for human motion generation that models motion on product manifolds and learns dynamics via Riemannian flow matching.
  \item We design a compact Riemannian representation that is both effective and efficient, and we provide a systematic evaluation of representation geometry in the context of human motion generation.
  \item To the best of our knowledge, we present the first demonstration that Riemannian flow matching scales effectively to large datasets and modern high-capacity generative architectures.
  \item We demonstrate strong empirical results on text-to-motion benchmarks, showing consistent gains across quality, alignment, and diversity metrics.
\end{itemize}

\begin{table}[htbp]
  \vspace{-1.0\baselineskip}
  \centering
  \footnotesize
  \caption{Comparison of motion representations. $\mathscr{T}$: global translation; $\mathscr{R}_{\mathrm{Orientation}}$: global orientation; $\mathscr{R}_{\mathrm{Joint}}$: per-joint rotation; $\mathscr{P}$: joint position; $\d\cdot$: temporal difference of the variable. Compared to existing formats, our method adopts the most concise format, with each factor represented on its natural manifold.\label{tab:format_comparison}}
  \begin{tblr}{
      width=\linewidth,
      colspec={Q[l] *{10}{Q[c]}},
      colsep=3pt,
      rowsep=2pt,
      hline{1}={0.08em},
      hline{3}={0.05em},
      hline{Z}={0.08em},
      cells={valign=m},
    }
    \SetCell[r=2]{c,m} \textbf{Format} & \SetCell[r=2]{c,m} $\mathscr{T}$ & \SetCell[r=2]{c,m} $\mathscr{R}_{\mathrm{Orientation}}$ & \SetCell[r=2]{c,m} $\mathscr{R}_{\mathrm{Joint}}$ & \SetCell[r=2]{c,m} $\mathscr{P}$ & \SetCell[r=2]{c,m} $\d \mathscr{T}$ & \SetCell[r=2]{c,m} $\d \mathscr{R}_{\mathrm{Orientation}}$ & \SetCell[r=2]{c,m} $\d \mathscr{R}_{\mathrm{Joint}}$ & \SetCell[r=2]{c,m} $\d \mathscr{P}$ & Foot & Ambient \\
     &  &  &  &  &  &  &  &  & Contact & Dimension \\
    HumanML3D & \ding{55} & \ding{55} & \ding{51} & \ding{51} &\ding{51} & \ding{51} &\ding{55}& \ding{51} & \ding{51} & $12J-1$ \\
    MotionStreamer& \ding{55} & \ding{55} & \ding{51} & \ding{51} &\ding{51} & \ding{51}&\ding{55} & \ding{51} & \ding{55} & $12J+8$ \\
    DART& \ding{51} & \ding{51} & \ding{51} & \ding{51} & \ding{51} &\ding{51}&\ding{55} & \ding{51} & \ding{55} & $12J+12$ \\
    HY-Motion& \ding{51} & \ding{51} & \ding{51} & \ding{51} &\ding{55} & \ding{55}&\ding{55} & \ding{55} & \ding{55} & $9J+3$ \\
    RMG (Ours)& \ding{51} & \ding{51} & \ding{51} & \ding{55} &\ding{55} & \ding{55}&\ding{55} & \ding{55} & \ding{55} & $4J+3$ \\
  \end{tblr}
  \vspace{-1.0\baselineskip}
\end{table}

\section{Related Works}
\subsection{Human Motion Generation}

Human motion generation has witnessed rapid progress, propelled by improvements in deep learning and motion capture technologies.
Early approaches primarily adopted regression models to map predefined action labels to motion sequences, exemplified by works such as Action2Motion\citep{guo2020action2motion} and SA-GAN \citep{yu2020structure}.
The introduction of advanced generative models—including Variational Autoencoders \citep{petrovich2021action,kingma2013auto}, Generative Adversarial Networks \citep{degardin2022generative,goodfellow2014generative}, normalizing flows \citep{rezende2015variational,valle2021transflower}, and more recently, diffusion models \citep{mdm, petrovich2021action, mld,tseng2023edge,kim2023flame,mofusion} has enabled the modeling of complex, multi-modal motion distributions and the synthesis of more realistic and diverse human motions.

A significant trend in recent years is the use of conditional generative models, where motion is synthesized based on various forms of contextual signals \citep{petrovich2022temos,kim2022brand,mdm,guo2020action2motion,petrovich2022temos,t2m_gpt,li2021ai,ao2022rhythmic}.
These methods have evolved from simple mappings to architectures that exploit joint embeddings, transformers\citep{petrovich2022temos,guo2020action2motion}, and diffusion processes \citep{mdm,tseng2023edge,mofusion,mld}, resulting in higher fidelity and controllable motion sequences.

Despite these advances, human motion generation remains challenging due to the highly articulated and nonlinear nature of human movement, as well as the need for semantic alignment with conditioning signals.
Evaluation protocols are still evolving, with a combination of objective metrics and user studies commonly employed to assess the naturalness, diversity, and consistency of generated motions \citep{humanml3d,mdm, petrovich2021action,chen2021choreomaster,huang2020dance,liu2022disco}.

\subsection{Riemannian Manifold}

A smooth manifold $\mathcal{M}$ \citep{lee_smooth_manifolds} is a topological space that locally resembles Euclidean space $\mathbb{R}^n$, which allows for the application of calculus, and it becomes a \emph{Riemannian manifold} $(\mathcal{M}, g)$ when it is endowed with a Riemannian metric $g$ \citep{lee_riemannian_manifolds}.
The metric is a smoothly varying inner product on each tangent space, defining the inner product between any two tangent vectors $u, v \in T_p\mathcal{M}$ as $g_p(u, v)$.
This structure allows for the measurement of geometric properties, such as the length of a vector, and the angle between vectors.
The metric also induces a distance function on the manifold, typically the geodesic distance, which measures the length of the shortest path between two points \citep{lee_riemannian_manifolds}.

Furthermore, the Riemannian metric defines a canonical volume form $\d V$, which is essential for integration and defining probability distributions on the manifold \citep{lee_riemannian_manifolds}.
A probability density function $f: \mathcal{M} \rightarrow \mathbb{R}^+$ must satisfy the normalization condition $\int_{\mathcal{M}} f(p) \d V(p) = 1$.
This has enabled the development of statistical models on non-Euclidean domains, such as the Riemannian uniform/normal distribution, which are crucial for modern data analysis as well as the generative modeling on the Riemannian manifold \citep{pennec2006,said2017}.


\subsection{General Flow Matching}
Flow matching \citep{flow_matching, stochastic_interpolant, reflow}, combining aspects from Continuous Normalizing Flows and Diffusion Models, learns a time-dependent velocity field that transports a source distribution $p_0$ to a target distribution $p_1$. A standard training objective is the conditional flow-matching loss:
\begin{equation}\label{eq:cfm}
  \mathcal{L}_{\mathrm{CFM}}(\theta)=\mathbb{E}_{t,\bm x_1,\bm x_t} \left\|v_t (\bm x_t|\bm x_1)-v_{\theta}(\bm x_t,t)\right\|_2^2
\end{equation}
where $v_t(\bm x_t \mid \bm x_1)$ denotes the target conditional velocity.
In Euclidean space, this target is typically induced by the linear interpolation between $\bm x_0$ and $\bm x_1$, which yields a simple closed-form velocity field.
Recent work extends the same principle to Riemannian manifolds \citep{riemann_flow_matching,meta_fm}, replacing Euclidean linear paths with geodesics and defining target velocities in the corresponding tangent spaces.
Consistently, the geodesic degrades to the linear interpolation when the manifold is the Euclidean space, which means that the Riemannian flow matching is a strict generalization of the Euclidean case.
These formulations show that flow matching is not restricted to flat Euclidean domains and provide the foundation for geometry-aware generative modeling.

\section{Methodology\label{sec:methodology}}

In this section, we present our proposed method, RMG.
\paragraph{Notation.} Unless specified otherwise, $\mathcal{M}$ denotes a Riemannian manifold.
For $\bm{x}\in\mathcal{M}$, $T_{\bm{x}}\mathcal{M}$ denotes the tangent space at $\bm{x}$, and $T\mathcal{M}$ denotes the tangent bundle.
We write $\mathrm{Exp}_{\bm{x}}: T_{\bm{x}}\mathcal{M}\rightarrow \mathcal{M}$ and $\mathrm{Log}_{\bm{x}}:\mathcal{M}\rightarrow T_{\bm{x}}\mathcal{M}$ for the exponential and logarithm maps, respectively.
For embedded manifolds, $\Pi_{T_{\bm{x}}\mathcal{M}}(\cdot)$ denotes the projection operator onto $T_{\bm{x}}\mathcal{M}$.

\subsection{Motion Representation and Manifold \label{sec:preliminary}}

Prior work typically factorizes a single motion frame into several parts as illustrated in \Cref{tab:format_comparison}. We adopt this decomposition but cast each factor on its natural Riemannian manifold, yielding a scale-free representation with intrinsic normalization. Unlike previous works that still apply dataset-level mean/standard-deviation normalization after forming the representation, the manifold structure (unit quaternions, pre-shapes, and a canonical length for translation) makes normalization unnecessary and enables geometry-aware modeling.

\paragraph{Global Translation ($\mathscr{T}$).}
Following common practice, we choose a specific joint of human body (usually chosen as the root joint, pelvis) to represent the global translation, which is a simple Euclidean space of $\mathbb{R}^3$.
This factor captures the global trajectory of the motion and is essential for modeling locomotion and spatial movement.

\paragraph{Global Orientation and Per-Joint Rotations ($\mathscr{R}$).}
The articulated rotations can be represented by unit quaternions. For a skeleton with $J$ joints, we write $q=\{q_j\}_{j=1}^{J}$ with $q_j\in\mathbb{H}$ and $\|q_j\|_2=1$.
Inspired by the SMPL parameterization \citep{SMPL:2015}, we define a canonical reference pose (typically the T-pose) and express all rotations relative to it: $q_1$ encodes the global orientation \((\mathscr{R}_{\mathrm{Orientation}})\), while $\{q_j\}_{j=2}^{J}$ capture the local joint rotations in their respective local coordinate systems \((\mathscr{R}_\mathrm{Joint})\). Since each $q_j$ is a unit quaternion, it lies on the hypersphere $\mathbb{S}^3$ (embedded in $\mathbb{R}^4$), and the rotation component lies on the product manifold $(\mathbb{S}^3)^J$.

Unlike continuous 6D rotations \citep{cont6d}, unit quaternions represent $\mathrm{SO}(3)$ without redundancy and induce smooth geodesics on $\mathbb{S}^3$. This improves interpolation and sampling stability, avoids re-orthogonalization, and reduces dimensionality from 6 to 4 with a consistent distance metric.

\paragraph{Local Pose ($\mathscr{P}$).}
We represent the within-frame skeletal configuration as a point in the Kendall pre-shape space\citep{kendall1984_shape_manifolds, shape_theory, dryden2016shape}.
Treating the $J$ joints as landmarks in $\mathbb{R}^3$, the pre-shape is invariant to global translation and scale and thus captures only the relative joint configuration within each frame. Concretely, $\mathcal{S}^J_3$ denotes the set of centered configurations with unit Frobenius norm.

Given joint coordinates $P\in\mathbb{R}^{J\times 3}$, we remove global translation by centering and remove scale by Frobenius normalization:
\begin{equation}
  p=\frac{P - \bar{P}}{\|P - \bar{P}\|_F}\in\mathcal{S}^J_3,\quad \bar{P} = \frac{1}{J}\sum_{j=1}^J P_j.
  \label{eq:preshape_formulation}
\end{equation}
Unlike variants that only subtract the root (or XZ-plane) translation, this pre-shape is fully translation-invariant and scale-free, making it well-suited for modeling the local pose factor.

\paragraph{Temporal Differentiation ($\d \cdot$).}
For each factor, we can also include its temporal difference (e.g., velocity) as an additional component. Specifically, we can compute the temporal difference by $\d \bm x_t = \mathrm{Log}_{\bm x_t}(\bm x_{t+1}) \in T_{\bm x_t}\mathcal{M}_{\mathscr{F}}$ for $\mathscr{F}\in\{\mathscr{T},\mathscr{R},\mathscr{P}\}$. This captures the dynamic aspect of motion.

\paragraph{Unified View.}
\Cref{fig:illustration} illustrates our formulation, which provides us a unified view of motion representation as a product manifold of the composite factors and covers all the representation in the previous works.
Take the HumanML3D \citep{humanml3d} format without the foot concact indicator as an example, it corresponds to a product manifold of the form:
\[\mathcal{M}_{\mathrm{H3D}} = T \mathcal{M}_{\mathscr{T}} \times T \mathcal{M}_{\mathscr{R}_{\mathrm{Orientation}}} \times \mathcal{M}_{\mathscr{R}_{\mathrm{Joint}}} \times \mathcal{M}_{\mathscr{P}} \times   T \mathcal{M}_{\mathscr{P}}\]

However, using so many factors is redundant and not necessary.
In this work, we argue global translation ($\mathscr{T}$), global orientation, and joint rotations ($\mathscr{R}$) are sufficient to capture articulated motion, which are verified through empirical studies (\Cref{sec:experiments}) and theoretical analysis (\Cref{sec:theory}).
We therefore adopt a more compact representation that omits the pre-shape and other temporal differences, yielding
\begin{equation}
  \mathcal{M}_{\mathrm{RMG}} = \mathcal{M}_{\mathscr{T}} \times \mathcal{M}_{\mathscr{R}} = \mathbb{R}^3\times (\mathbb{S}^3)^J,
  \label{eq:rmg_manifold}
\end{equation}

We will elaborate and justify this choice in the ablation study (\Cref{sec:ablation}).

\begin{figure}[tbp]
  \centering
  \includegraphics[width=0.9\textwidth, trim=0 50 270 0, clip]{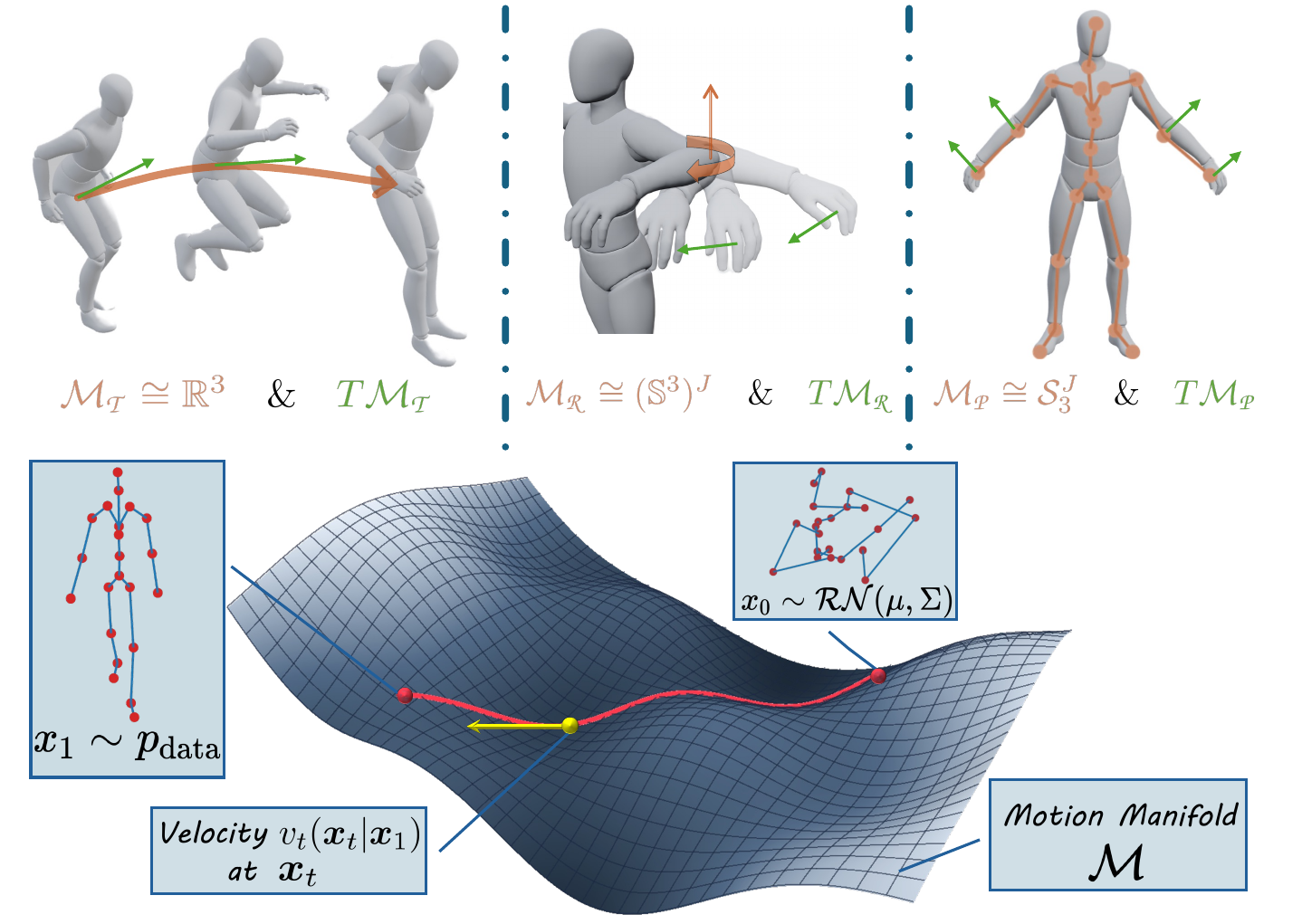}
  \caption{(\textbf{Top}) Illustration of the unified Riemannian representation for articulated motion. Each motion frame can be factorized into \textcolor[RGB]{205,144,114}{\textbf{global translation} $(\mathcal{M}_{\mathscr{T}})$}, \textcolor[RGB]{205,144,114}{\textbf{global orientation and per-joint rotations} $(\mathcal{M}_{\mathscr{R}})$}, and \textcolor[RGB]{205,144,114}{\textbf{local pose} $(\mathcal{M}_{\mathscr{P}})$} along with the \textcolor[RGB]{101,165,66}{\textbf{temporal differences} $(T\mathcal{M}_{\mathscr{F}}\ \text{for}\ \mathscr{F}\in\{\mathscr{T},\mathscr{R},\mathscr{P}\}$)}. Each factor is represented on its natural Riemannian manifold, yielding a scale-free representation with intrinsic normalization. (\textbf{Bottom}) Illustration of the Riemannian flow matching process in the RMG manifold. \(\mathcal{M}\) is defined by our proposed manifold \Cref{eq:rmg_manifold}. The red line is the geodesic between \(\bm x_0\) and \(\bm x_1\) while the yellow line with arrow is the velocity at \(\bm x_t\).\label{fig:illustration}}
\end{figure}

\subsection{Prior Distribution}
\label{sec:prior}

\paragraph{Riemannian Gaussian Distribution.}
We first introduce a mean-centered wrapped Gaussian distribution on a general Riemannian manifold $\mathcal{M}$. Given a reference (mean) point $\mu\in\mathcal{M}$, we draw Gaussian noise in its embedded Euclidean space $\mathbb{R}^n$, map it to the tangent space $T_{\mu}\mathcal{M}$ (via a projection operator), and then ``wrap'' it onto the manifold using the exponential map:
\[\xi \sim \mathcal{N}(0,\Sigma)\in \mathbb{R}^n,\qquad v= \Pi_{T_{\mu}\mathcal{M}}(\xi) \in T_{\mu}\mathcal{M}, \qquad z = \mathrm{Exp}_{\mu}\big(v\big)\in\mathcal{M},\]
where $\Sigma$ is typically chosen block-diagonal. Such distribution can be denoted as \(\mathcal{RN}(\mu, \Sigma)\).

\paragraph{Choice of Reference Point.}
The reference point $\mu$ can be chosen arbitrarily, but a good choice improves sampling quality.
For motion generation, we set $\mu$ to the \emph{rest pose} with zero translation and identity rotations, which is a natural center of the motion manifold.
This choice ensures that samples from the prior correspond to plausible static poses, providing a meaningful starting point for motion synthesis.

Specifically, for the translation factor, we set $\mu_{\mathscr{T}}=\bm 0 \in\mathbb{R}^3$; for the rotation factor, we set $\mu_{\mathscr{R}}=[1,0,0,0]\in \mathbb{S}^3$; for the pre-shape factor (if used), we set $\mu_{\mathscr{P}}$ to the canonical T-pose after removing global translation and normalizing using \Cref{eq:preshape_formulation}.
In our chosen manifold $\mathcal{M}_{\mathrm{RMG}}$ (\Cref{eq:rmg_manifold}), the reference point is thus
\[\mu=(\mu_{\mathscr{T}}, \underbrace{\mu_{\mathscr{R}}, \mu_{\mathscr{R}}, \dots, \mu_{\mathscr{R}}}_{J \text{ times}})\in \mathcal{M}_{\mathrm{RMG}}.\]

\subsection{Training and Inference\label{sec:train}}
We train a time-dependent vector field on the motion manifold $\mathcal{M}$ with Riemannian flow matching.
Let $\bm{x}_1\sim p_{\mathrm{data}}$ denote a real motion sample and $\bm{x}_0\sim p_0$ a prior sample (\Cref{sec:prior}).
For $t\sim\mathcal{U}[0,1]$, we construct the interpolation state on the geodesic from $\bm{x}_0$ to $\bm{x}_1$:
\begin{equation*}
  \bm{x}_t \;=\; \mathrm{Exp}_{\bm{x}_0}\!\big(t\,\mathrm{Log}_{\bm{x}_0}(\bm{x}_1)\big),
  \label{eq:geodesic_interpolation}
\end{equation*}
On product manifolds, $\mathrm{Exp}$ and $\mathrm{Log}$ are applied factor-wise.
The translation factor is Euclidean, while rotation and (optional) pre-shape factor use the corresponding manifold maps.
The supervision signal is the geodesic tangent at $\bm{x}_t$, written as
\begin{equation*}
  v_t(\bm{x}_t\,|\,\bm{x}_1) \;=\; \frac{1}{1-t}\,\mathrm{Log}_{\bm{x}_t}(\bm{x}_1)\;\in\;T_{\bm{x}_t}\mathcal{M},
  \label{eq:target_velocity}
\end{equation*}
which reduces to the standard Euclidean flow-matching target when $\mathcal{M}=\mathbb{R}^n$.

We parameterize the vector field by a neural network $v_{\theta}(\bm{x}_t,t)$. Since the output of the neural network is in the ambient Euclidean space, we must project it to the tangent space to enforce valid manifold dynamics:
$\Pi_{T_{\bm{x}_t}\mathcal{M}}v_{\theta}(\bm{x}_t,t)\in T_{\bm{x}_t}\mathcal{M}$.

Training minimizes the mean-squared error between target and predicted tangent velocities:
\begin{equation}
  \mathcal{L}(\theta)
  \;=\;
  \mathbb{E}_{\bm{x}_1\sim p_{\mathrm{data}},\,\bm{x}_0\sim p_0,\,t\sim \mathcal{U}[0,1]}
  \Big[
    \big\|v_t(\bm{x}_t\,|\,\bm{x}_1) - \Pi_{T_{\bm{x}_t}\mathcal{M}} v_{\theta}(\bm{x}_t,t)\big\|_2^2
  \Big].
  \label{eq:training_loss}
\end{equation}

At inference time, we sample $\bm{x}_0\sim p_0$ and integrate the learned manifold ODE:
\begin{equation}
  \frac{\mathrm{d}\bm{x}_t}{\mathrm{d}t} \;=\; \Pi_{T_{\bm{x}_t}\mathcal{M}} v_{\theta}(\bm{x}_t,t),\qquad \bm{x}_0\sim p_0,
  \label{eq:inference_ode}
\end{equation}
from $t=0$ to $t=1$.
With step size $h$, a first-order Riemannian Euler update is
$\bm{x}_{t+h}=\mathrm{Exp}_{\bm{x}_t}\!\big(h\,\Pi_{T_{\bm{x}_t}\mathcal{M}}v_{\theta}(\bm{x}_t,t)\big)$,
which preserves manifold constraints by construction.

We emphasize that \(T\mathcal{M}_{\mathscr{F}}\), with \(\mathscr{F}\in\{\mathscr{T},\mathscr{R},\mathscr{P}\}\), in \Cref{sec:preliminary} and \(T_{\bm{x}_t}\mathcal{M}\) in \Cref{sec:train} refer to different objects despite the similar notation. The former appears in the factorized motion representation and denotes optional temporal-difference components attached to data samples, which therefore belong to the data distribution \(p_{\mathrm{data}}=p_1\). The latter is the tangent space at the interpolation state \(\bm{x}_t\) and contains the time-dependent velocity field used by Riemannian flow matching, corresponding to the intermediate probability flow \(p_t\).

\section{Experiments}
\label{sec:experiments}

\subsection{Experiment Setup}

\paragraph{Datasets.} The experiments are mainly conducted on \emph{HumanML3D} \citep{humanml3d}.
HumanML3D is a large-scale language-motion dataset comprising 14,616 motions and 44,970 text descriptions building upon the AMASS dataset \citep{AMASS:ICCV:2019}.
Besides HumanML3D, we also employ \emph{MotionMillion} \citep{motionmillion}, which is a recently released large-scale motion dataset with 1 million motion clips and 4 million text descriptions, to pre-train our model and evaluate its generalization ability.

\paragraph{Evaluation Metrics.} Following previous works \citep{humanml3d,mld,mdm}, we employ 4 main metrics to evaluate our framework. The Frechet Inception Distance is used to measure the motion quality as well as the feature distributions.
The Diversity and MultiModality Distance are incorporated to measure the generation diversity. Lastly, the R Precision aims to evaluate the matching rate of the conditions.

\paragraph{Implementation Details.} We use the Diffusion Transformer \citep{dit} as the model backbone for the flow matching.
For text encoding, we incorporate Qwen \citep{qwen3embedding,qwen3}, utilizing the encoded hidden states as the text representation.
During training, we employ the AdamW \citep{adamw} optimizer with a cosine learning-rate schedule and linear warmup: the learning rate is first warmed up linearly to $10^{-4}$ and then cosine-annealed to near zero.
Training is performed in a classifier-free manner \citep{classifier_free_guidance} with a dropout rate of $10\%$.
To stabilize both training and inference, we adopt an exponential moving average (EMA) strategy.

\subsection{Main Results}

\begin{table}[htbp]
  \vspace{-1.0\baselineskip}
  \centering
  \footnotesize
  \caption{Evaluation of \textit{text-based motion generation} on HumanML3D \citep{humanml3d} dataset. Reported metrics are FID, R@1, Diversity, and MModality. The models in bold are the optimal models, and the models in underline are the sub-optimal models. The guidance scale is set to $6.5$.\label{tab:h3d_results}}
  \begin{tblr}{
      width=\linewidth,
      colspec={Q[l] *{4}{Q[c]} *{4}{Q[c]}},
      row{1,2}={font=\bfseries},
      colsep=3pt,
      rowsep=1pt,
      hline{1}={0.08em},
      hline{2}={2-5}{0.05em, rightpos=2.5},
      hline{2}={6-9}{0.05em, leftpos=2.5},
      hline{4}={0.04em},
      hline{13}={0.04em},
      row{13}={bg=gray!15},
      hline{3}={0.05em},
      hline{Z}={0.08em}
    }
    \SetCell[r=2]{c,m} Method & \SetCell[c=4]{c} HumanML3D Format & & & & \SetCell[c=4]{c} MotionStreamer Format & & & \\
    & FID$\downarrow$ & R@1$\uparrow$ & Div$\rightarrow$ & MModality$\uparrow$ & FID$\downarrow$ & R@1$\uparrow$ & Div$\rightarrow$ & MModality$\uparrow$ \\
    GT & 0.002 & 0.511 & 9.503 & 2.799 & 0.002 & 0.702 & 27.492 & -- \\
    MLD \citeyearpar{mld} & 0.473 & 0.481 & 9.724 & 2.413 & 18.236 & 0.546 & 26.352 & -- \\
    T2M-GPT\citeyearpar{t2m_gpt} & 0.116 & 0.492 & 9.761 & 1.856 & 12.475 & 0.606 & 27.275 & -- \\
    MotionGPT \citeyearpar{mgpt} & 0.232 & 0.492 & \textbf{9.528} & 2.008 & 14.375 & 0.456 & 27.114 & -- \\
    MoMask  \citeyearpar{momask} & \underline{0.045} & 0.521 & -- & 1.241 & 12.232 & 0.621 & 27.127 & --\\
    MotionLCM \citeyearpar{motionlcm} & 0.304 & 0.505 & 9.607 & 2.259 & -- & -- & -- & -- \\
    MotionCLR \citeyearpar{motionclr} & 0.269 & \textbf{0.542} & 9.607 & 1.985 & -- & -- & -- & -- \\
    MotionLab \citeyearpar{motionlab} & 0.167 & -- & 9.593 & \textbf{2.912} & -- & -- & -- & -- \\
    MARDM \citeyearpar{mardm} & 0.114 & 0.500 & -- & 2.231 & -- & -- & -- & -- \\
    MotionStreamer \citeyearpar{motionstreamer} & -- & -- & -- & -- & \underline{11.790} & \underline{0.631} & \underline{27.284} & -- \\
    Ours & \textbf{0.043} &\underline{0.525} & \textbf{9.555} & \underline{2.748} & \textbf{5.835} & \textbf{0.710} & \textbf{27.672} & \textbf{14.906} \\
  \end{tblr}
\end{table}

We first evaluate our method on the HumanML3D text-to-motion generation benchmark and compare it with several strong baselines, including both diffusion-based methods (e.g., MotionDiffuse \citep{motion_diffuse}, MDM \citep{mdm}, MotionLab \citep{motionlab}) and autoregressive methods (e.g., T2M-GPT \citep{t2m_gpt}, MGPT \citep{mgpt}, MoMask \citep{momask}). We report results under two output formats: the standard HumanML3D format and the MotionStreamer format.
\emph{It is worth noting that we implement extra functions to convert our Riemannian representation to either the HumanML3D format or MotionStreamer format for fair comparison. Refer to supplementary materials \Cref{sec:conversion_functions} for details.}

\Cref{tab:h3d_results} reports quantitative results of the 4 metrics on HumanML3D under two formats, where we focus primarily on the standard HumanML3D format. In this setting, our method achieves the best FID (0.043), slightly surpassing the previous best MoMask (0.045), indicating stronger motion realism and distribution matching. At the same time, our model maintains strong text-motion consistency with R@1 = 0.525 (second only to MotionCLR at 0.542), while preserving high generation diversity and multimodality (Div = 9.555 and MModality = 2.748). Compared with prior methods that optimize only part of this trade-off (e.g., lower FID or higher R@1 alone), our model provides a more balanced improvement across quality, alignment, and diversity, which is critical for practical text-to-motion generation.
For completeness, under the MotionStreamer format our method ranks first among all reported metrics (FID = 5.835, R@1 = 0.710, Div = 27.672, MModality = 14.906), further supporting the robustness of the learned Riemannian representation across output formats.
We also provide additional results in the supplementary materials.

\begin{wraptable}{r}{0.48\textwidth}
  \vspace{-1.0\baselineskip}
  \centering
  \caption{Evaluation of \textit{text-based motion generation} on MotionMillion \citep{motionmillion} dataset. Reported metrics are FID and R@1. In our method, 0.5B and 1.7B refer to the parameter sizes of our Riemannian flow matching model and Qwen3-1.7B, respectively. G.S. refers to the guidance scale.\label{tab:mm_results}}
  \begin{tblr}{
      width=\linewidth,
      colspec={Q[l] Q[c] Q[c] Q[c]},
      row{1}={font=\bfseries},
      row{5}={bg=gray!15},
      row{6}={bg=gray!15},
      colsep=4pt,
      rowsep=2pt,
      hline{1}={0.08em},
      hline{2}={0.05em},
      hline{5}={0.04em},
      hline{Z}={0.08em},
    }
    Method & G.S. & FID$\downarrow$ & R@1$\uparrow$ \\
    ScaMo \citeyearpar{scamo} & -- & 89.0 & 0.67 \\
    MotionMillion-3B & -- & 10.8 & 0.79 \\
    MotionMillion-7B & -- & 10.3 & 0.79 \\
    Ours (0.5B+1.7B) & 2.0 & \textbf{5.6} & \underline{0.81} \\
    Ours (0.5B+1.7B) & 3.0 & \underline{7.8} & \textbf{0.86} \\
  \end{tblr}
\end{wraptable}

We further evaluate text-based motion generation on the large-scale MotionMillion benchmark to assess whether the advantages of our representation persist in a substantially broader data regime. \Cref{tab:mm_results} shows that our model (\emph{Ours, 0.5B+1.7B}) outperforms the MotionMillion baselines across both fidelity and text alignment, while also exhibiting a clear and favorable guidance trade-off. With guidance scale $2.0$, our method attains better FID of $5.6$, improving substantially over the strongest baseline MotionMillion-7B ($10.3$) and thus reducing the distribution gap by nearly half. Increasing the guidance scale to $3.0$ further raises R@1 from $0.81$ to $0.86$, while still maintaining a stronger FID than previous baselines.
These results are particularly notable because our approach remains superior to substantially when scaling up, indicating that the gains come not merely from model scaling, but from the effectiveness of the proposed Riemannian motion representation and flow-matching formulation.

Overall, these findings highlight the scalability and generalization ability of our \textbf{RMG} framework, which can be effectively applied to both small and large datasets and models while maintaining superior performance across key metrics.

\subsection{Ablation Studies\label{sec:ablation}}

In this section, we conduct ablation studies to analyze the impact of different factors in our framework, and we answer two main research questions:
\begin{enumerate}[label=\textbf{RQ\arabic*:}, ref=RQ\arabic*]
  \item \label{rq:1} Which factors matter for motion quality and guidance stability?
  \item \label{rq:2} Does temporal difference modeling improve motion quality?
\end{enumerate}

\begin{figure}[htbp]
  \centering
  \begin{subfigure}[t]{0.49\textwidth}
    \centering
    \includegraphics[width=\textwidth]{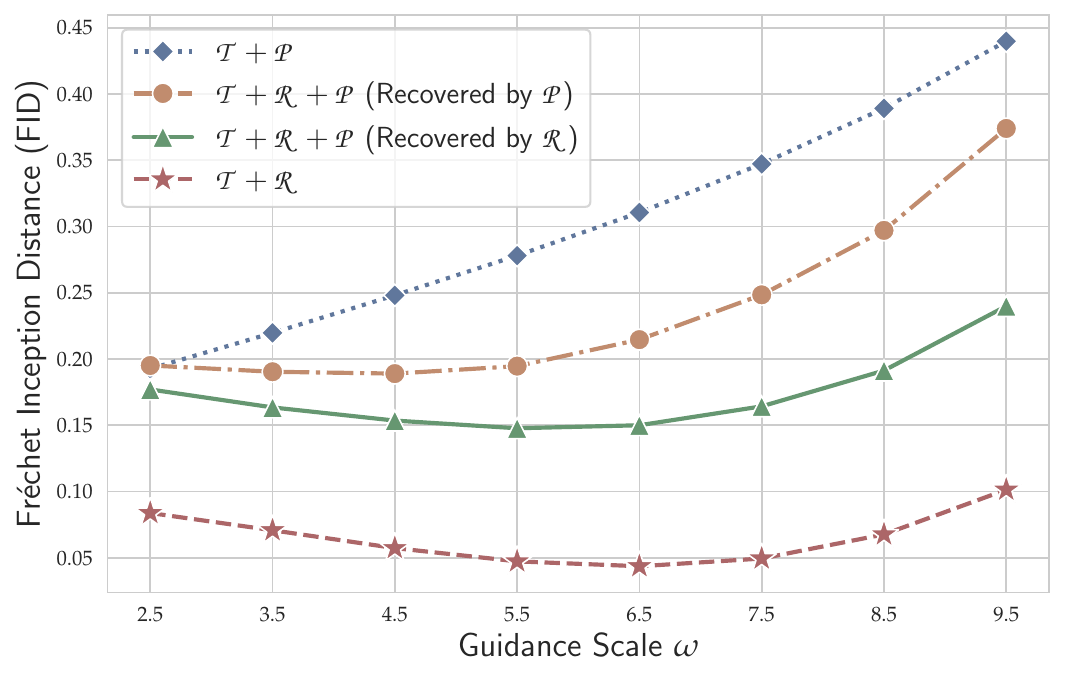}
    \caption{ We study the impact of the Riemannian representation including $\mathscr{T}+\mathscr{R}$, $\mathscr{T}+\mathscr{R}+\mathscr{P}$ and $\mathscr{T}+\mathscr{P}$.}
    \label{fig:ablation_v1}
  \end{subfigure}
  \hfill
  \begin{subfigure}[t]{0.49\textwidth}
    \centering
    \includegraphics[width=\textwidth]{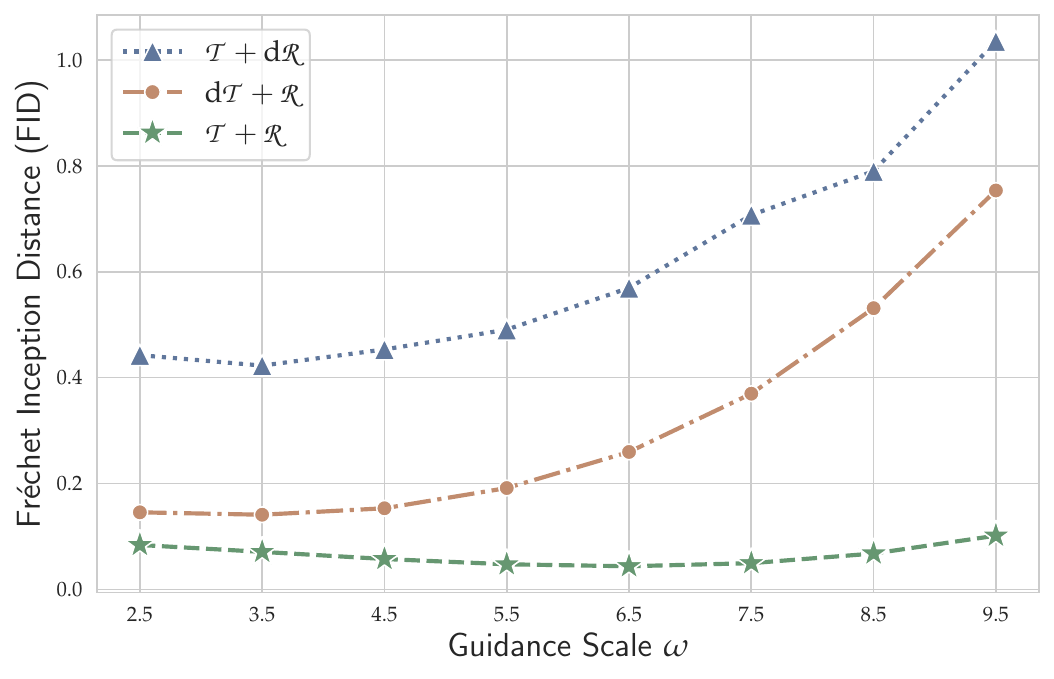}
    \caption{ We study the impact of temporal difference modeling in our framework including $\mathscr{T}+\mathscr{R},\ \d\mathscr{T}+\mathscr{R},\ \mathscr{T}+\d\mathscr{R}$. }
    \label{fig:ablation_v2}
  \end{subfigure}
  \caption{Ablation study on different factors of our framework. All the models are trained with the same setting (including the parameter size, random seed) and evaluated on the HumanML3D benchmark.\label{fig:ablation}}
  \vspace{-1.0\baselineskip}
\end{figure}

\paragraph{\ref{rq:1}.}To isolate the role of each factor, we keep the same architecture and training setup and only change the representation, then sweep the guidance scale $\omega \in [2.5, 9.5]$. \Cref{fig:ablation_v1} shows that $\mathscr{T}+\mathscr{R}$ is consistently the best and most stable setting: its FID decreases from approximately $0.084$ ($\omega=2.5$) to a minimum of about $0.043$ ($\omega=6.5$), and remains low even at large guidance ($0.101$ at $\omega=9.5$). In contrast, $\mathscr{T}+\mathscr{P}$ degrades monotonically as guidance increases (from $0.20$ to $0.44$), indicating poor robustness without rotation information.

For $\mathscr{T}+\mathscr{R}+\mathscr{P}$, the term `Recovered by \(\cdot\)' means different conversion functions used to convert to the HumanML3D format for evaluation. Refer to \Cref{sec:conversion_functions} for details. We find that recovering with $\mathscr{R}$ is consistently better than recovering with $\mathscr{P}$ across the entire guidance range (e.g., $0.043$ vs.\ $0.084$ at $\omega=2.5$, and $0.101$ vs.\ $0.44$ at $\omega=9.5$). This suggests that the rotation component is more critical for maintaining motion quality and stability under guidance, while the pose-coordinate component does not provide additional robustness in this setup.

This trend is also aligned with practical motion representations: in real animation and robotics pipelines, motion is predominantly driven by global/root translation and joint rotations. By contrast, pose-coordinate representations are more affected by subject-specific factors (e.g., body scale and limb proportions), which introduces additional variability across people. Taken together, these observations suggest that $\mathscr{T}+\mathscr{R}$ is a compact and robust representation that is sufficient for motion modeling in our setting.

\paragraph{\ref{rq:2}.}To answer whether temporal difference modeling improves motion quality, we compare $\mathscr{T}+\mathscr{R}$ with $\d\mathscr{T}+\mathscr{R}$ and $\mathscr{T}+\d\mathscr{R}$ under the same guidance sweep $\omega \in [2.5, 9.5]$. \Cref{fig:ablation_v2} shows that $\mathscr{T}+\mathscr{R}$ consistently attains the lowest FID across the full range of $\omega$. In contrast, $\d\mathscr{T}+\mathscr{R}$ degrades steadily as guidance increases ($0.14 \rightarrow 0.75$), and $\mathscr{T}+\d\mathscr{R}$ is less stable still ($0.44 \rightarrow 1.03$). This behavior is consistent with the representation property of temporal differencing: it attenuates absolute/global motion state (e.g., global trajectory and orientation) while emphasizing local frame-to-frame variation, which weakens long-horizon structure modeling. Overall, temporal differencing does not improve performance in our setting, and absolute modeling of translation and rotation (\mbox{$\mathscr{T}+\mathscr{R}$}) remains the most robust choice for guidance-stable generation.

\section{Conclusion}

\paragraph{Limitations.} Though our work has demonstrated promising results, there are certain limitations including lack of further exploration of more condition modalities and the potential for facial and hand motion modeling. We leave more discussion in the supplementary materials \Cref{sec:limitations}.

In conclusion, we proposed \textbf{Riemannian Motion Generation (RMG)}, a unified geometric framework for human motion representation and generation. Instead of treating motion as an unconstrained Euclidean vector, we modeled it on a product manifold and learned generation dynamics with Riemannian flow matching. This design yields a compact representation, $\mathbb{R}^3\times(\mathbb{S}^3)^J$, together with a geometry-consistent training and inference pipeline based on geodesic interpolation, tangent-space vector-field learning, and manifold-preserving integration.

Extensive experiments demonstrated that this geometric formulation is both effective and scalable. RMG achieves strong and balanced performance on HumanML3D across motion quality, text alignment, and diversity, and generalizes well to the larger MotionMillion benchmark. Our ablation studies further showed that translation and rotation factors are sufficient in practice, while adding extra factors does not necessarily improve generation quality. Overall, the results support representation geometry as a core design principle for motion generation. 

{\small
  \bibliographystyle{plainnat}
  \bibliography{main}
}

{\newpage
\appendix

\noindent
\textbf{\Large Appendix}

\section{Manifold}

In this paper, we focus on three types of Riemannian manifolds: Euclidean space $\mathbb{R}^3$, the hypersphere $\mathbb{S}^3$, and the pre-shape space \(\mathcal{S}^J_3\).

\paragraph{Euclidean space.} Euclidean space $\mathbb{R}^3$ is a flat manifold with zero curvature. It is the most familiar manifold, where the distance between two points is simply their straight-line distance. Most existing diffusion models are built in this space.

\paragraph{Hypersphere.} The hypersphere \(\mathbb{S}^3\) is a curved manifold with constant positive curvature. It is embedded in \(\mathbb{R}^4\) and consists of points at a fixed distance from the origin. In this work, we use the hypersphere to model quaternions.
Since \(q\) and \(-q\) represent the same rotation, we restrict quaternions to the upper hemisphere of \(\mathbb{S}^3\) to avoid ambiguity during dataset construction.
In other words, for any point \(q = (w, x, y, z) \in \mathbb{S}^3\), we have \(q_0 > 0\).

For Riemannian flow matching on a general hyper-sphere \(\mathcal{S}^d\), consider two points \(\bm x_0, \bm x_1 \in \mathbb{R}^{d+1}\) on the hypersphere \(\mathcal{S}^d\). The geodesic between them can be written as
\[\gamma(t) = \frac{\sin((1-t)\theta)}{\sin(\theta)}\bm x_0 + \frac{\sin(t\theta)}{\sin(\theta)}\bm x_1,\]
where \(\theta = \arccos(\langle \bm x_0, \bm x_1 \rangle)\) is the angle between \(\bm x_0\) and \(\bm x_1\). The geodesic velocity at time \(t\) is given by:
\[\dot{\gamma}(t) = \frac{\theta}{\sin(\theta)}\left(-\sin(t\theta)\bm x_0 + \sin((1-t)\theta)\bm x_1\right).\]

\paragraph{Pre-shape space.} The pre-shape space models the relative configuration of points. For example, the set of all triangles in \(\mathbb{R}^2\) forms a pre-shape space, denoted by \(\mathcal{S}^3_2\). In this work, we use the pre-shape space \(\mathcal{S}^J_3\) to model the human skeleton, which consists of \(J\) joints in \(\mathbb{R}^3\). A pre-shape is represented by a \(J \times 3\) matrix, where each row gives the coordinates of one joint. The pre-shape space is defined as the set of all such matrices that are centered at the centroid and normalized to unit Frobenius norm. A concrete formulation is given in \Cref{eq:preshape_formulation}. Intrinsically, the pre-shape space can be viewed as a high-dimensional hypersphere, so geodesics can be computed in the same way as on the hypersphere.

Beyond the pre-shape space, one can also consider the shape space, which is the quotient of the pre-shape space by the rotation group: \(\mathcal{S}^J_3 \backslash \text{SO}(3)\). However, geodesic computation in the shape space is substantially more involved, so we only study the pre-shape space in this work.

For Riemannian flow matching on a general pre-shape space \(\mathcal{S}_m^k\), for any two pre-shapes $\bm X_0, \bm X_1\in\mathcal{S}_m^k$, define
\[
  \theta = \arccos \text{trace}(\bm X_0^\top \bm X_1), \qquad 0 \le \theta \le \pi .
\]

If $\bm X_0 \neq \pm \bm X_1$, the unique minimizing geodesic connecting $\bm X_0$ and $\bm X_1$ on $\mathcal{S}_m^k$ is given by
\[
  \Gamma(t)
  =
  \frac{\sin((1-t)\theta)}{\sin\theta}\,\bm X_0
  +
  \frac{\sin(t\theta)}{\sin\theta}\,\bm X_1,
  \qquad t \in [0,1].
\]
and the geodesic velocity at time \(t\) is given by
\[
  \dot{\Gamma}(t)
  =
  \frac{\theta}{\sin\theta}\left(-\sin(t\theta)\bm X_0 + \sin((1-t)\theta)\bm X_1\right).
\]

\section{Theoretical Support}
\label{sec:theory}

\subsection{Redundancy induces flat directions}

\begin{proposition}[Redundancy induces flat directions]
  Let $\mathcal{M}$ be a smooth $m$-dimensional manifold and let $\pi: U \subset \mathbb{R}^d \to \mathcal{M}$ be a smooth submersion with $d > m$.
  For any loss $L(z) = \ell(\pi(z))$, the gradient $\nabla L(z)$ is orthogonal to $\ker(D\pi(z))$, and the Hessian $\nabla^2 L(z)$ has at least $d - m$ zero eigenvalues at any critical point.
  In other words, representing an $m$-dimensional quantity in $\mathbb{R}^d$ introduces $d-m$ flat directions in the loss landscape.
\end{proposition}

\begin{proof}
  \textbf{Gradient claim.}
  By the chain rule, the differential of $L$ at $z$ satisfies
  \[
    \d L_z = \d\ell_{\pi(z)} \circ D\pi(z).
  \]
  For any $v \in \ker(D\pi(z))$, we have $D\pi(z)v = 0$, so
  \[
    \nabla L(z)^\top v = \d L_z(v) = \d\ell_{\pi(z)}(D\pi(z)v) = \d\ell_{\pi(z)}(0) = 0.
  \]
  Hence $\nabla L(z)$ is orthogonal to every redundant direction.

  \textbf{Hessian claim.}
  Consider the curve $z(t) = z + tv$ and set $x(t) = \pi(z(t))$.
  Differentiating at $t=0$ gives $x'(0) = D\pi(z)v = 0$.
  Differentiating $L(z(t)) = \ell(x(t))$ twice and evaluating at $t=0$ yields
  \[
    v^\top \nabla^2 L(z)\, v
    = \frac{\d^2}{\d t^2} L(z(t)) \bigg|_{t=0}
    = \nabla \ell(x(0))^\top x''(0) + x'(0)^\top \nabla^2 \ell(x(0))\, x'(0).
  \]
  The second term vanishes because $x'(0)=0$. At a critical point of $L$, the gradient claim above gives $\nabla L(z)=0$, which by the chain rule forces $\nabla \ell(\pi(z))=0$, so the first term also vanishes. Therefore $v^\top \nabla^2 L(z)\,v = 0$ for all $v \in \ker(D\pi(z))$.

  \textbf{Zero eigenvalue count.}
  Since $\pi$ is a submersion, $\operatorname{rank} D\pi(z) = m$, and rank-nullity gives $\dim\ker(D\pi(z)) = d-m$.
  The Hessian $\nabla^2 L(z)$ is symmetric, and every vector in this $(d-m)$-dimensional subspace is a zero eigenvector, so $\nabla^2 L(z)$ has at least $d-m$ zero eigenvalues.
\end{proof}

This proposition shows that when a representation contains redundant coordinates, the loss landscape contains intrinsically flat directions.
These directions correspond to variations that do not change the underlying state on the manifold, which may lead to an ill-posed optimization problem and unstable training dynamics.
On the contrast, a compact representation aligned with the intrinsic manifold removes these redundant degrees of freedom.

\subsection{Statistical advantage of Riemannian Flow Matching}

\begin{proposition}[Informal statistical advantage of Riemannian Flow Matching]
  Let $(\mathcal{M},g)$ be a compact $d$-dimensional Riemannian manifold embedded in $\mathbb{R}^D$ with $d<D$, and suppose the data distribution is supported on $\mathcal{M}$.
  Let $X_t\in \mathcal{M}$ be a conditional interpolation, and let
  \[
    \rho_t := \mathrm{Law}(X_t)\in \mathcal P(\mathcal{M})
  \]
  denote its marginal distribution at time $t$.
  Define the target data distribution by
  \[
    \rho_{\mathrm{data}} := \rho_1.
  \]
  Assume $\dot X_t\in T_{X_t}\mathcal{M}$, and define the conditional flow matching target by
  \[
    u_t^\star(x)=\mathbb E[\dot X_t \mid X_t=x].
  \]
  Assume further that $u^\star$ belongs to an $s$-smooth class of tangent vector fields on $\mathcal{M}$.
  Then the Riemannian flow matching estimator $\hat u$ satisfies
  \[
    \mathbb E\!\left[\int_0^1 \|\hat u_t-u_t^\star\|_{L^2(\rho_t;T\mathcal{M})}^2\,dt\right]
    \lesssim
    n^{-\frac{2s}{2s+d}}.
  \]
  Moreover, if $\hat \rho_{\mathrm{data}}$ denotes the generated distribution at terminal time $t=1$, then under standard flow stability assumptions,
  \[
    \mathbb E\,W_2^2(\hat \rho_{\mathrm{data}},\rho_{\mathrm{data}})
    \lesssim
    n^{-\frac{2s}{2s+d}}.
  \]

  By contrast, for Euclidean Flow Matching in the ambient space $\mathbb R^D$, one typically obtains
  \[
    \mathbb E\,W_2^2(\hat \rho_{\mathrm{data}}^{\,\mathrm{EFM}},\rho_{\mathrm{data}})
    \lesssim
    n^{-\frac{2s}{2s+D}}.
  \]
  Since $d<D$, we have
  \[
    n^{-\frac{2s}{2s+d}}
    \gg
    n^{-\frac{2s}{2s+D}},
  \]
  and therefore the error bound of Riemannian Flow Matching is asymptotically strictly better than that of Euclidean Flow Matching.
\end{proposition}

\begin{proof}[Proof sketch]
  The flow matching population loss is
  \[
    \mathcal L(u)=\mathbb E\big[\|u_t(X_t)-\dot X_t\|_g^2\big].
  \]
  By the usual regression identity,
  \[
    \mathcal L(u)-\mathcal L(u^\star)
    =
    \int_0^1 \|u_t-u_t^\star\|_{L^2(\rho_t;T\mathcal{M})}^2\,dt.
  \]
  Therefore, estimating the flow matching vector field reduces to a nonparametric regression problem on the intrinsic state space where the samples live.

  If the model is formulated intrinsically on $\mathcal{M}$, then the relevant function class is an $s$-smooth class on a $d$-dimensional manifold.
  Standard entropy bounds on manifolds imply that its statistical complexity depends on the intrinsic dimension $d$, which yields the rate
  \[
    n^{-\frac{2s}{2s+d}}.
  \]
  In contrast, a naive Euclidean formulation in $\mathbb R^D$ uses a $D$-dimensional ambient smooth class, leading to the slower rate
  \[
    n^{-\frac{2s}{2s+D}}.
  \]
  Finally, the terminal generated distribution is obtained by pushing the reference distribution through the learned flow up to time $t=1$.
  Standard stability of the flow map implies that the error in the learned vector field controls the terminal distribution error, hence
  \[
    \mathbb E\,W_2^2(\hat \rho_{\mathrm{data}},\rho_{\mathrm{data}})
    \lesssim
    n^{-\frac{2s}{2s+d}}.
  \]
  Combining the two bounds and using $d<D$ shows that Riemannian Flow Matching has a strictly sharper asymptotic error estimate than Euclidean Flow Matching.
\end{proof}

This proposition provides a theoretical justification for the statistical superiority of Riemannian Flow Matching over Euclidean Flow Matching when the data distribution is supported on a low-dimensional manifold embedded in a high-dimensional ambient space.
When the data distribution lies on a \(d\)-dimensional manifold, learning flow matching in a manifold representation reduces the effective dimension of the regression problem for the velocity field. As a consequence, the Wasserstein error of the learned generative distribution scales with \(d\) rather than the ambient dimension \(D\), leading to improved asymptotic rates.

\section{Additional Experiment Results}

\subsection{Full Results}

\begin{table}[htbp]
  \centering
  \footnotesize
  \caption{Full results on the HumanML3D dataset in H3D format.\label{tab:h3d_full_results_h3d_format}}
  \begin{tblr}{
      width=\linewidth,
      colspec={Q[l] *{7}{Q[c]}},
      row{1,2}={font=\bfseries},
      colsep=3pt,
      rowsep=1pt,
      hline{1}={0.08em},
      hline{2}={2-8}{0.05em},
      hline{4}={0.04em},
      hline{13}={0.04em},
      row{13}={bg=gray!15},
      hline{3}={0.05em},
      hline{Z}={0.08em}
    }
    \SetCell[r=2]{c,m} Method & \SetCell[c=7]{c} HumanML3D Format & & & & & & \\
    & FID$\downarrow$ & R@1$\uparrow$ & R@2$\uparrow$ & R@3$\uparrow$ & MM-Dist$\downarrow$ & Div$\rightarrow$ & MM$\uparrow$ \\
    GT & 0.002 & 0.511 & 0.703 & 0.797 & 2.974 & 9.503 & 2.799 \\
    MLD & 0.473 & 0.481 & 0.673 & 0.772 & 3.196 & 9.724 & 2.413 \\
    T2M-GPT & 0.116 & 0.492 & 0.679 & 0.775 & 3.121 & 9.761 & 1.856 \\
    MotionGPT & 0.232 & 0.492 & 0.681 & 0.733 & 3.096 & \textbf{9.528} & 2.008 \\
    MoMask & \underline{0.045} & 0.521 & \underline{0.713} & 0.807 & 2.958 & -- & 1.241 \\
    MotionGPT-2 & 0.191 & 0.496 & 0.691 & 0.782 & 3.080 & 9.860 & 2.137 \\
    MotionLCM & 0.304 & 0.505 & 0.705 & 0.805 & 2.986 & 9.607 & 2.259 \\
    MotionCLR & 0.269 & \textbf{0.544} & \textbf{0.732} & \textbf{0.831} & \textbf{2.806} & 9.607 & 1.985 \\
    MotionLab & 0.167 & -- & -- & \underline{0.810} & \underline{2.830} & 9.593 & \textbf{2.912} \\
    MARDM & 0.114 & 0.500 & 0.695 & 0.795 & -- & -- & 2.231 \\
    Ours & \textbf{0.043}\(^{\pm .002}\) & \underline{0.525}\(^{\pm .002}\) & 0.711\(^{\pm .002}\) & \(0.805^{\pm .002}\) & \( 2.930^{\pm .007} \) & \underline{9.555}\(^{\pm .060}\) & \underline{2.748}\(^{\pm .023}\) \\
  \end{tblr}
\end{table}

\begin{table}[htbp]
  \centering
  \footnotesize
  \caption{Full results on the HumanML3D dataset in MotionStreamer format. \label{tab:h3d_full_results_ms_format}}
  \begin{tblr}{
      width=\linewidth,
      colspec={Q[l] *{7}{Q[c]}},
      row{1,2}={font=\bfseries},
      colsep=3pt,
      rowsep=1pt,
      hline{1}={0.08em},
      hline{2}={2-8}{0.05em},
      hline{4}={0.04em},
      hline{9}={0.04em},
      row{9}={bg=gray!15},
      hline{3}={0.05em},
      hline{Z}={0.08em}
    }
    \SetCell[r=2]{c,m} Method & \SetCell[c=7]{c} MotionStreamer Format & & & & & & \\
    & FID$\downarrow$ & R@1$\uparrow$ & R@2$\uparrow$ & R@3$\uparrow$ & MM-Dist$\downarrow$ & Div$\rightarrow$ & MM$\uparrow$ \\
    GT & 0.002 & 0.702 & 0.864 & 0.914 & 15.151 & 27.492 & -- \\
    {\scriptsize MLD} & 18.236 & 0.546 & 0.730 & 0.792 & 16.638 & 26.352 & -- \\
    {\scriptsize T2M-GPT} & 12.475 & 0.606 & 0.774 & 0.838 & 16.812 & 27.275 & -- \\
    {\scriptsize MotionGPT} & 14.375 & 0.456 & 0.598 & 0.628 & 17.892 & 27.114 & -- \\
    {\scriptsize MoMask} & 12.232 & 0.621 & 0.784 & 0.846 & 16.138 & 27.127 & -- \\
    {\scriptsize MotionStreamer} & \underline{11.790} & \underline{0.631} & \underline{0.802} & \underline{0.859} & \underline{16.081} & \underline{27.284} & -- \\
    Ours & \textbf{5.835}\(^{\pm .060}\) & \textbf{0.710}\(^{\pm .002}\) & \textbf{0.850}\(^{\pm .002}\) & \textbf{0.899}\(^{\pm .001}\) & \textbf{15.683}\(^{\pm .022}\) & \textbf{27.672}\(^{\pm .076}\) & \textbf{14.906}\(^{\pm .208}\) \\
  \end{tblr}
\end{table}

\begin{table}[htbp]
  \centering
  \caption{Full results of MotionMillion dataset. \label{tab:mm_full_results}}
  \begin{tblr}{
      width=\linewidth,
      colspec={Q[l] Q[c] Q[c] Q[c] Q[c] Q[c] Q[c]},
      row{1}={font=\bfseries},
      row{5}={bg=gray!15},
      row{6}={bg=gray!15},
      colsep=4pt,
      rowsep=2pt,
      hline{1}={0.08em},
      hline{2}={0.05em},
      hline{5}={0.04em},
      hline{Z}={0.08em},
    }
    Method & Guidance Scale & FID$\downarrow$ & R@1$\uparrow$ & R@2$\uparrow$ & R@3$\uparrow$ \\
    ScaMo & -- & 89.0 & 0.67 & 0.81 & 0.87 \\
    MotionMillion-3B & -- & 10.8 & 0.79 & 0.91 & 0.94 \\
    MotionMillion-7B & -- & \underline{10.3} & \underline{0.79} & 0.90 & 0.94 \\
    Ours (0.5B+1.7B) & 2.0 & \textbf{5.6} & \underline{0.81} & \underline{0.92} & \underline{0.95} \\
    Ours (0.5B+1.7B) & 3.0 & \underline{7.8} & \textbf{0.86} & \textbf{0.95} & \textbf{0.97} \\
  \end{tblr}
\end{table}

We report the full results of our method and all baselines on HumanML3D in both the original H3D format (\Cref{tab:h3d_full_results_h3d_format}) and the MotionStreamer format (\Cref{tab:h3d_full_results_ms_format}), together with the full results on MotionMillion (\Cref{tab:mm_full_results}).
For HumanML3D, each experiment is repeated over 20 independent runs, and we report the mean with 95\% confidence intervals computed using the \(t\)-distribution.
For MotionMillion, the dataset scale and our compute budget limit evaluation to a single run.

\subsection{Training Stability}

\begin{figure}[htbp]
  \centering
  \includegraphics[width=0.9\linewidth]{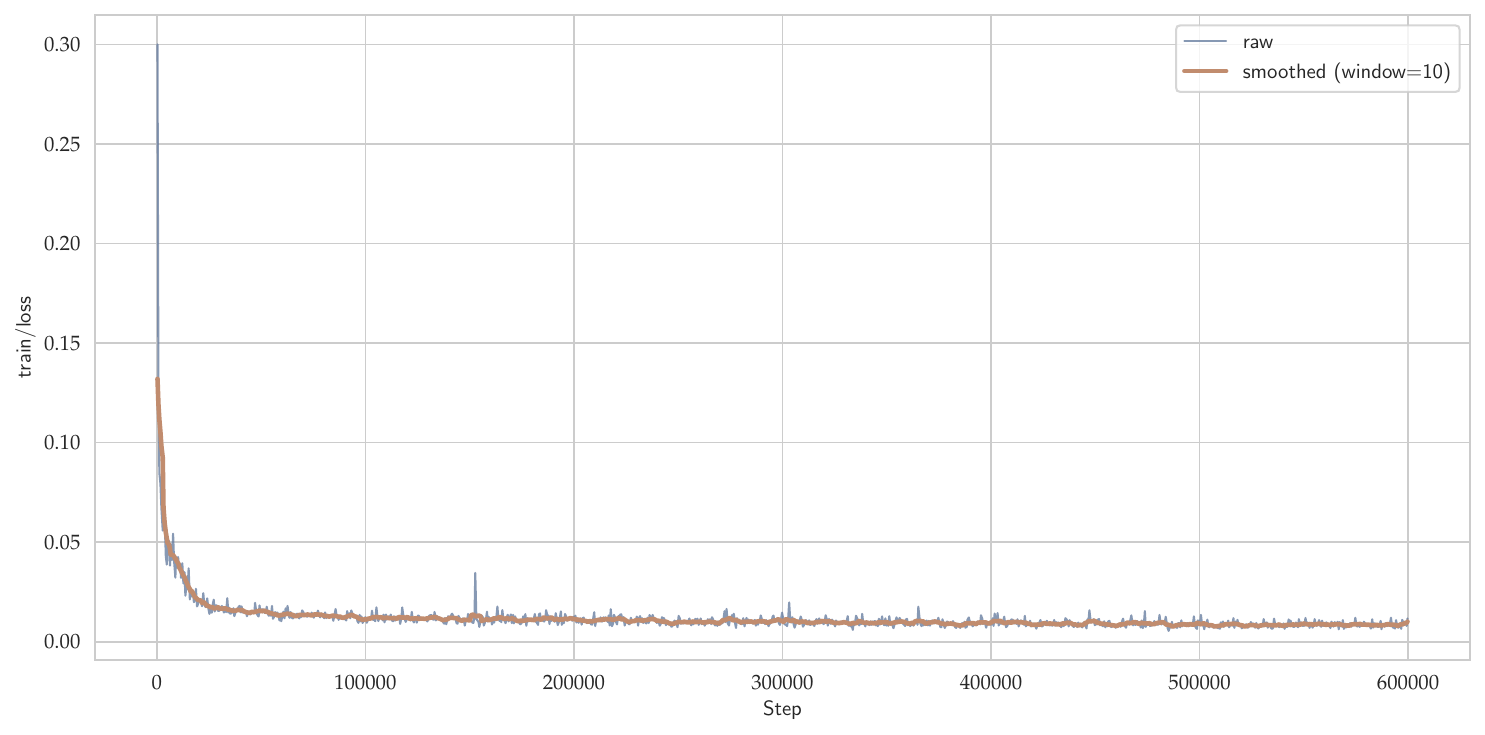}
  \caption{Training loss curves for MotionMillion.\label{fig:mm_loss}}
\end{figure}

\begin{figure}[htbp]
  \centering
  \includegraphics[width=0.9\linewidth]{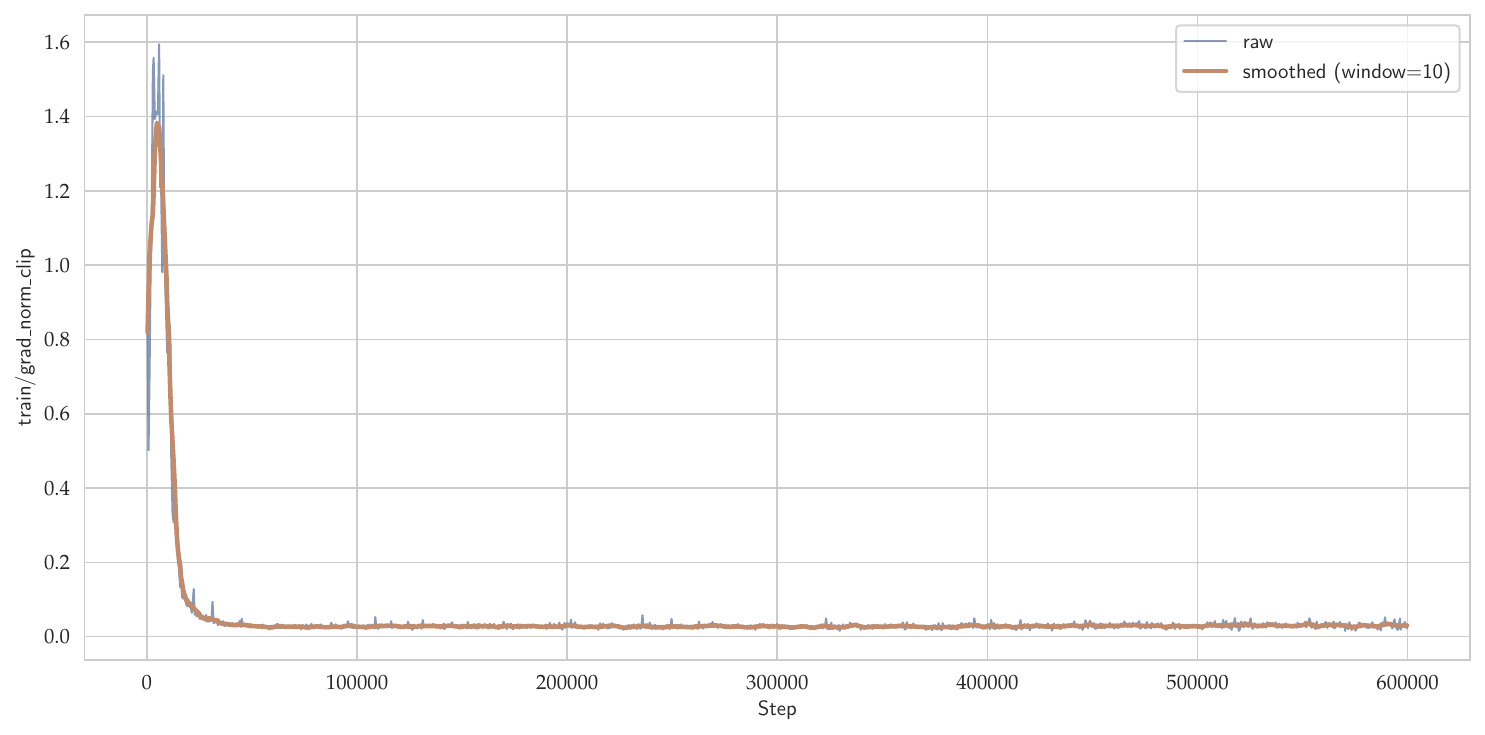}
  \caption{Gradient norm curves for MotionMillion.\label{fig:mm_grad_norm}}
\end{figure}

We monitor the training loss and gradient norm curves for our MotionMillion experiments, as shown in \Cref{fig:mm_loss} and \Cref{fig:mm_grad_norm}.
The loss curve shows a smooth and steady decrease throughout training.
The gradient norm curve also remains stable with only one spike in the beginning of training, and that's why we adopt a warm-up learning rate strategy and a single gradient clipping threshold for the entire training process.
And at the end of training, both the loss and gradient norm curves are still stable, which suggests that the training process is well-behaved and does not diverge or collapse.

\section{Implementation Details}

\subsection{Model Architecture}

Our model uses a Diffusion Transformer backbone, while the text encoder varies across datasets.

For HumanML3D, we use Qwen3-Embedding-0.6B \citep{qwen3embedding} to extract a 1024-dimensional text feature vector.
We fuse the text features with the time embedding through an MLP and use the fused representation as the conditioning input to the Diffusion Transformer.

For MotionMillion, we use Qwen3-1.7B \citep{qwen3} to extract text features.
Unlike Qwen3-Embedding-0.6B, Qwen3-1.7B is a decoder-only large language model.
Following recent advances in image generation \citep{zimage,sd3}, we therefore adopt a single-stream multi-modal Diffusion Transformer (MM-DiT), where text and motion tokens are concatenated along the sequence dimension and processed as a unified input stream.

\subsection{Training Details}

\begin{table}[htbp]
  \centering
  \footnotesize
  \caption{Training hyperparameters for different model scales.\label{tab:train_hyper_param}}
  \begin{tblr}{
      width=0.72\linewidth,
      colspec={Q[l,1.6] Q[c,1] Q[c,1]},
      row{1}={font=\bfseries},
      colsep=4pt,
      rowsep=2pt,
      hline{1}={0.08em},
      hline{2}={0.05em},
      hline{Z}={0.08em},
      cells={valign=m},
    }
    Hyperparameter & RMG-base & RMG \\
    Input Dim. & 91 & 91 \\
    Hidden Dim. & 384 & 1024 \\
    No. of Layers & 6 & 24 \\
    No. of Heads & 8 & 8 \\
    FFN Multiplier & 8 & 4 \\
    Max Learning Rate & 1e-4 & 1e-4 \\
    Learning Rate Scheduler & Cosine w/ warmup & Cosine w/ warmup \\
    Ratio of Warmup Steps & 0.08 & 0.08 \\
    Effective Batch Size & \(32\times 8 \times 1\) & \(16 \times 8 \times 2\) \\
    Training Steps & 150k & 600k \\
    Gradient Clipping & 0.5 & 0.5 \\
  \end{tblr}
\end{table}

We list the training hyperparameters for different model scales in \Cref{tab:train_hyper_param}.
The effective batch size is calculated as the product of the per-device batch size, the number of devices, and the gradient accumulation steps.

\subsection{Conversion Functions}
\label{sec:conversion_functions}

\begin{figure}[htbp]
  \centering
  \includegraphics[width=0.9\linewidth]{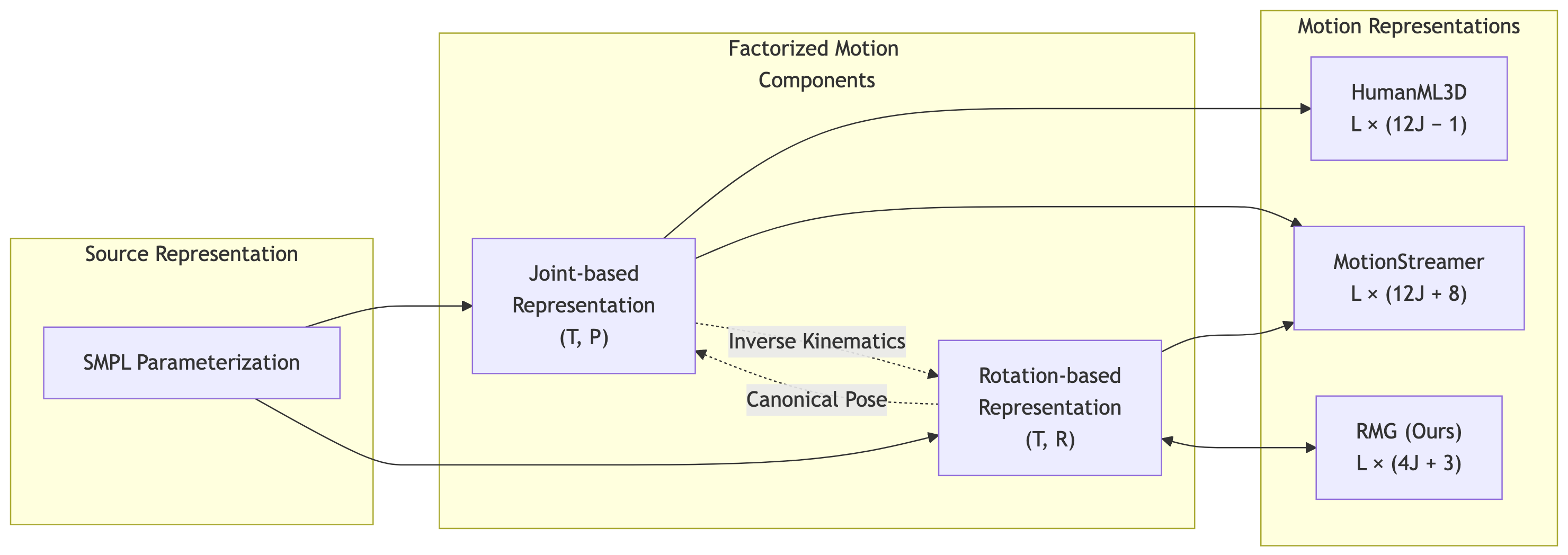}
  \caption{The conversion functions between different motion representations.\label{fig:conversion_functions}}
\end{figure}

We implement two extra conversion functions to convert our representation of human motion to HumanML3D format and MotionMillion format, respectively.
The core idea is to first map our representation back to either a rotation-based or a joint-based representation, and then apply the same preprocessing pipeline used by the target format.
\Cref{fig:conversion_functions} summarizes these conversion functions.

\section{Limitations}
\label{sec:limitations}

While our proposed Riemannian Motion Generation framework demonstrates promising results in motion representation and generation, there are several limitations that warrant discussion:

\begin{itemize}
  \item The Riemannian representation has not been tested on the auto-regressive setting though it's been a popular stream in recent human motion generation works.
  \item More conditioning modalities (e.g., music, video) and interactive generation scenarios (e.g., human-object interaction) have not been explored.
  \item The generation duration is currently limited to 10 seconds (300 frames), and scaling to longer horizons may require larger data collection efforts and more computational resources.
  \item More body configurations (e.g., hands and face) are not included in the current framework, and extending to these richer configurations may require additional design considerations.
  \item Motion editing and temporal in-painting are not studied in this work, and it remains to be seen how the Riemannian representation can be adapted to these tasks.
\end{itemize}

\section{Broader Impacts}

This work may have several positive societal impacts.
First of all, it may continue to drive research on large-scale human motion generation.
Secondly, it may enable more efficient and scalable motion generation systems, which can be beneficial for various applications such as gaming, virtual reality, and human-computer interaction.
Thirdly, it may also inspire future research on geometric representations and low-dimensional manifold modeling in other domains beyond human motion.
Lastly, more advanced topics such as world modeling and embodied agents may also benefit from the rapid development of human motion generation.

At the same time, this line of research may also introduce potential risks.
More capable human motion generation systems could be misused to synthesize deceptive or misleading human behaviors, and biases or artifacts in the training data may be inherited or amplified in generated motions.
In addition, deploying such models in embodied or interactive systems without sufficient safeguards may create safety, fairness, or reliability concerns in downstream applications.
We therefore encourage future work to pair technical advances with careful dataset curation, clear disclosure of synthetic content, and application-level safety mechanisms.
}{}

\end{document}